\documentclass[journal]{IEEEtran}
\usepackage[noadjust]{cite}
\usepackage{adjustbox}
 \usepackage{multirow}
\usepackage{wrapfig}
\usepackage{epsfig,psfrag}
\usepackage{latexsym,amssymb,amsmath}
\usepackage{longtable}
\usepackage{tikz,pgf}
\usepackage{wrapfig}
\usepackage{rotating}
\usepackage{float}
\usepackage{stfloats}
\usepackage{graphicx}
\usepackage{epstopdf}
\usepackage{color}
\usepackage{setspace}

\usepackage{amsmath,amsfonts}
\newcommand{\bqn}{\begin{eqnarray}}
\newcommand{\eqn}{\end{eqnarray}}
\newcommand{\bq}{\begin{eqnarray*}}
\newcommand{\eq}{\end{eqnarray*}}

\usepackage{amsthm}
\theoremstyle{definition}

\theoremstyle{plain}
\newtheorem{theorem}{Theorem}

\usepackage{url}
\usepackage{hyperref}
\hypersetup{colorlinks,%
citecolor=black,%
filecolor=black,%
linkcolor=black,%
urlcolor=black}%

\usepackage[normalem]{ulem}

\usepackage{cleveref}   

\usepackage{supertabular}   

\begin{document}
\bstctlcite{IEEEexample:BSTcontrol}

\title{Fast Polynomial Approximation of Heat Kernel Convolution on Manifolds and Its Application to Brain Sulcal and Gyral Graph Pattern Analyis}
\author{Shih-Gu Huang, Ilwoo Lyu, Anqi Qiu, and Moo K. Chung
\thanks{Copyright (c) 2019 IEEE. Personal use of this material is permitted. However, permission to use this material for any other purposes must be obtained from the IEEE by sending a request to pubs-permissions@ieee.org.} 
\thanks{This work was supported by NIH Grant R01 EB022856 (Correspondence author: Moo Chung. e-mail: mkchung@wisc.edu)}
\thanks{
S.-G. Huang and M.K. Chung are with the Department of Biostatistics and Medical Informatics, University of Wisconsin, Madison, WI 53706, USA. }
\thanks{I. Lyu is with the Department of Electrical Engineering and Computer Science, Vanderbilt University, Nashville, TN 37235, USA} 
\thanks{A. Qiu is with the Department of Biomedical Engineering, National University of Singapore, Singapore 117583}
}

\maketitle

\begin{abstract}
Heat diffusion has been widely used in brain imaging for surface fairing, mesh regularization and  cortical data smoothing. Motivated by diffusion wavelets and convolutional neural networks on graphs, we present a new fast and accurate numerical scheme to solve heat diffusion on surface meshes. This is achieved by approximating the heat kernel convolution using high degree orthogonal polynomials in the spectral domain. We also derive the closed-form expression of the spectral decomposition of the Laplace-Beltrami operator and use it to solve heat diffusion on a manifold for the first time.
The proposed fast polynomial approximation scheme avoids solving for the eigenfunctions of the Laplace-Beltrami operator, which is computationally costly for large mesh size, and the numerical instability associated with the finite element method based diffusion solvers. 
The proposed method  is applied in localizing the male and female differences in cortical sulcal and gyral graph patterns obtained from MRI in an innovative way. The MATLAB code is available at \url{http://www.stat.wisc.edu/~mchung/chebyshev}.
\end{abstract}

\begin{IEEEkeywords}
Heat diffusion, Laplace-Beltrami operator, brain cortical sulcal curves, diffusion wavelets, Chebyshev polynomials.
\end{IEEEkeywords}

\section{Introduction}
\IEEEPARstart{H}{eat} diffusion has been widely used in brain image processing as a form of smoothing and noise reduction starting with  Perona and Malik's ground-breaking study \cite{perona.1990}. 
Many techniques have been developed for surface mesh fairing, regularization \cite{sochen.1998,malladi.2002} and surface data smoothing \cite{andrade.2001,cachia.TMI.2003,chung.2001.diffusion,joshi.2009}. 
The diffusion equation has been solved by various numerical techniques \cite{andrade.2001,chung.2001.diffusion,cachia.TMI.2003,seo.2010.MICCAI,chung.2015.MIA}.
In \cite{chung.2001.diffusion,chung.2003.cvpr,chung.2004.ISBI}, the heat diffusion was solved iteratively by the discrete estimate of the LB-operator using  the finite element method (FEM) and the FDM. However, the FDM are known to suffer numerical instability if the sufficiently small step size is not chosen in the forward Euler scheme. In \cite{qiu.2006,reuter.2009.cad,seo.2010.MICCAI,chung.2015.MIA}, diffusion was solved by expanding the heat kernel as a series expansion 
of the LB-eigenfunctions. Although the LB-eigenfunction approach avoids the numerical instability associated with the FEM based diffusion solver  \cite{chung.2003.cvpr},  the computational complexity of computing eigenfunctions is very high for large-scale surface meshes. 

In this paper, motivated by the diffusion wavelet transform \cite{hammond.2011,coifman2006diffusion,kim2012wavelet,tan.2015,donnat2018learning} and convolutional neural networks \cite{defferrard2016convolutional} on graphs that all use Chebyshev polynomials, we propose a new spectral method to solve the heat diffusion by approximating  the heat kernel by orthogonal polynomials. The previous works did the spectral decomposition on mostly graph Laplacian exclusively using Chebyshev polynomials. The LB-operator with other polynomials were not considered before. We present a new general theory for the LB-operator on an arbitrary manifold that works with an arbitrary orthogonal polynomial. Besides the Chebyshev polynomials, we provide three other polynomials to show the generality of the proposed method. We further derive the closed-form expression of the spectral decomposition of the LB-operator and use it to solve heat diffusion on a manifold for the first time.
Taking the advantage of the recurrence relations of orthogonal polynomials \cite{chihara2011introduction,freud2014orthogonal,olver2010nist}, the computational run time  of the  proposed method  is 
significantly reduced.
The proposed method is 
faster than the LB-eigenfunction approach and FEM based diffusion solvers \cite{chung.2015.MIA}.  We further applied the fast polynomial approximation method to iterative convolution to obtain multiscale features, which is shown to be as good as  the diffusion wavelet  in detecting localized surface signals \cite{coifman2006diffusion,hammond.2011,kim2012wavelet,tan.2015,donnat2018learning}.

The proposed method is applied in quantifying brain sulcal and gyral patterns. The sulcal and gyral features such as gyrification index, sulcal depth, curvature, sulcal length and sulcal area  were widely used in revealing significant differences between populations \cite{im2019sulcal}. In \cite{ochiai2004sulcal}, the difference of the superior temporal sulcus length was analyzed. \cite{shi2017conformal} computed the lengths of sulcal curves in characterizing the Alzheimer's disease  (AD). \cite{seong2010automatic} measured the sulcal depth and average mean curvature along the sulcal lines in the AD study. In \cite{meng2016discovering}, various metrics were proposed to measure the difference between sulcal graph features including sulcal pits, sulcal basins and ridge points.
\cite{lyu2018cortical} computed local gyrification index using shape-adaptive kernels by performing  wavefront propagation with sulcal and gyral  curves as the source.\cite{im2011quantitative} measured the similarity between two sulcal graphs.

The main contributions of the paper are as follows. 1) The development of a general polynomial approximation theory for 
LB-operator and heat kernel and its application to solving diffusion equations fast. The derivation of the closed-form solutions of the expansion that enables faster computation of heat diffusion than before. 2) New multiscale shape analysis framework  on manifolds that utilizes the iterative heat kernel convolution property that is as powerful as diffusion wavelets. 3) Application of the faster solver in quantifying the sulcal and gyral patterns on the large-scale brain surface meshes with 370,000 vertices for 444 subjects obtained from 3T MRI.  We use the proposed method in performing diffusion on cortical brain surfaces by taking the sulcal and gyral graph patterns as the initial condition. The dataset is large enough to demonstrate the effectiveness of our faster solver. Our fast solver can perform diffusion in 40 minutes for the whole dataset. The male and female differences  are then localized using both mass univariate and multivariate statistics.

\section{Methods}

We present a new general spectral theory for diffusion equations and the heat kernel using four different types of polynomials (Jacobi, Chebyshev, Hermite, Laguerre) to show the generality of  the method. The analytic closed-form solutions to the expansion coefficients are derived and used to solve the heat diffusion fast. The new theory works for an arbitrary orthogonal polynomial.

\subsection{Diffusion on Manifolds}
Let functional data $f \in L^2(\mathcal{M})$, the space of square integrable functions on manifold
$\mathcal{M}$ with inner product 
\begin{equation*}
\langle f,h\rangle=\int_{\cal M}f(p)h(p) d\mu(p),
\end{equation*}
where $\mu(p)$ is the Lebesgue measure such that $\mu({\mathcal M})$ is the total area or volume of $\mathcal M$. Let $\Delta$ denote the LB-operator on $\mathcal{M}$. Let $\psi_j$ be the eigenfunctions of the LB-operator with eigenvalues $\lambda_j$, i.e., 
${\Delta}\psi_j=\lambda_j\psi_j$.
Let us order the eigenvalues as $0=\lambda_0 \leq \lambda_1 \leq \lambda_2 \leq
\cdots$.

The isotropic heat diffusion on $\mathcal{M}$ with $f$ as the initial observed data is given by
\begin{equation}
\frac{\partial g(p,\sigma)}{\partial \sigma}+\Delta g=0, \quad 
g(p,\sigma=0)=f(p),
\label{eq:cauchy}
\end{equation}
where $\sigma$ is the diffusion time.
It has been shown that the convolution of $f$ with heat kernel $K_\sigma$ is the unique solution of \eqref{eq:cauchy}
\cite{rosenberg.1997,chung.2005.IPMI,seo.2010.MICCAI,chung.2015.MIA}:
\begin{equation*}
g(p,\sigma)=K_\sigma \ast f (p)=\int_{\mathcal M} K_\sigma(p,q) f(q)d\mu(q),
\end{equation*}
with the heat kernel given by 
\begin{equation}\label{eq:heatkernel}
K_{\sigma}(p,q)=\sum_{j=0}^{\infty} e^{-\lambda_j \sigma}\psi_j(p)\psi_j(q).
\end{equation}
The heat kernel convolution can be written as
\begin{equation}\label{eq:smooth_eig}
g(p,\sigma) = 
K_\sigma \ast f (p)
=\sum_{j=0}^{\infty}e^{-\lambda_j \sigma}f_j\psi_j(p)
\end{equation}
with coefficients $f_j$ computed as 
\begin{equation*}
f_j=\int_{\mathcal{M}} f(p) \psi_j (p) d\mu(p).
\end{equation*}

\subsection{Basic Idea in 1D Diffusion}
We explain the core idea with 1D example. Consider time series data $f$ on $[0,1]$. The solution of heat diffusion \eqref{eq:cauchy} is given by  the weighted cosine series representation \cite{chung.2007.TMI},
\begin{equation}\label{eq:1Dwcsr}
g(p, \sigma) =K_\sigma \ast f =\sum_{j=0}^\infty e^{- j^2\pi^2 \sigma}f_j\psi_j, 
\end{equation}
where $\psi_0(p) =1$ and 
$\psi_j(p) = \sqrt{2} \cos (j\pi p)$ are the eigenfucntions of $\Delta = -\frac{\partial^2}{\partial p^2}$. From Taylor expansion $e^z=\sum_{n=0}^\infty \frac{z^n}{n!}$, 
\begin{equation*}
K_\sigma \ast f=\sum_{n=0}^\infty \frac{(-\sigma)^n}{n!} \sum_{j=0}^\infty f_j ( j^2\pi^2)^n\psi_j.
\end{equation*}
Since $\Delta\psi_j = j^2\pi^2\psi_j$ and $\Delta^n\psi_j=(j^2\pi^2)^n\psi_j$, we have
\begin{equation*}
K_\sigma \ast f =\sum_{n=0}^\infty \frac{(- \sigma)^n}{n!} \Delta^n f.
\end{equation*}
Thus, heat diffusion can be computed simply by using the power of Laplacian. If we can further compute the power of Laplacian quickly using some recursion, the computation can be done more quickly.

\subsection{Fast Polynomial Approximation}
Here we present a general new theory for an arbitrary manifold that works in any type of image domain including surface and volumetric meshes. Consider an  orthogonal polynomial  $P_n$ over  interval $[a,b]$ with inner product 
$\int_{a}^{b} P_n(\lambda) P_k(\lambda) w(\lambda) d \lambda = \delta_{nk},$
the Dirac delta. The weight $w(\lambda)$ differs for polynomials. $P_n$ is often defined using the second order recurrence 
\cite{olver2010nist},
\begin{equation}\label{eq:recurrence}
P_{n+1}(\lambda) = (A_n\lambda +B_n)P_{n} (\lambda) +C_n P_{n-1} (\lambda)
\end{equation}
with initial conditions $P_{-1}(\lambda)=0$ and  $P_0({\lambda})=1$. We expand the exponential weight
of the heat kernel by polynomials $P_n$:
\begin{equation}
e^{-\lambda \sigma}=\sum_{n=0}^\infty c_{\sigma,n}P_n(\lambda), 
c_{\sigma,n}= \int_{a}^{b}  e^{-\lambda \sigma} P_n(\lambda) w(\lambda) d \lambda.
\label{eq:expweight}\end{equation}
Substituting  \eqref{eq:expweight} into  \eqref{eq:smooth_eig}, the solution of heat diffusion can be expressed in terms of the polynomials:
\begin{equation}\label{eq:heat_Pn}
K_\sigma \ast f
=\sum_{n=0}^\infty c_{\sigma,n} \sum_{j=0}^{\infty} P_n(\lambda_j) f_j\psi_j.
\end{equation}
Since ${\Delta} \psi_j=\lambda_j\psi_j$, we have $\Delta^l\psi_j =\lambda^l_j\psi_j.$
Assuming the form $P_n(\lambda)=\sum_{l=0}^nd_l\lambda^l$,  we have
\begin{equation}\label{eq:property1}
P_n(\lambda_j)\psi_j=\sum_{l=0}^nd_l\lambda_j^l\psi_j=\sum_{l=0}^nd_l{\Delta}^l \psi_j=P_n(\Delta)\psi_j.
\end{equation}
By substituting \eqref{eq:property1} into \eqref{eq:heat_Pn}, 
the heat diffusion equation is solved by polynomial expansion
involving the LB-operator but without the LB-eigenfunctions,
\begin{equation*}
K_\sigma \ast f=\sum_{n=0}^\infty c_{\sigma,n}P_n\left(\Delta\right) f.
\end{equation*}
Since $P_n$ is a polynomial of degree $n$, the direct computation of $P_n\left(\Delta\right) f$ requires the costly computation of $\Delta f, \Delta^2 f, \cdots, \Delta^n f$. 
Instead, we compute $P_n\left(\Delta\right) f$ by the following recurrence
\begin{equation*}
P_{n+1}\left(\Delta\right) f = (A_n\Delta +B_n)P_{n}\left(\Delta\right) f + C_n P_{n-1}\left(\Delta\right) f
\end{equation*}
with initial conditions $P_{-1}(\Delta) f=0$ and $P_0(\Delta) f =f$.

In practice, the expansion is truncated at degree $m$, which is empirically determined. The expansion coefficients $c_{\sigma,n}$ can be computed from the closed-form solution to \eqref{eq:expweight}. In the following, we present three examples of the fast polynomial approximation methods based on the Jacobi, Hermite and Laguerre polynomials.

\noindent{\textbf{Jacobi polynomials.}}
The Jacobi polynomials $P_n^{(\alpha,\beta)}(\lambda)$, which are orthogonal in $[-1, 1]$ for $\alpha,\beta>-1$, are defined by the  recurrence \eqref{eq:recurrence} with parameters given by \cite{olver2010nist},
\begin{align*}
A_n&=\frac{(2n+\alpha+\beta+1)(2n+\alpha+\beta+2)}{2(n+1)(n+\alpha+\beta+1)},\nonumber\\
B_n&=\frac{(\alpha^2-\beta^2)(2n+\alpha+\beta+1)}{2(n+1)(n+\alpha+\beta+1)(2n+\alpha+\beta)},\\
C_n&=-\frac{(n+\alpha)(n+\beta)(2n+\alpha+\beta+2)}{(n+1)(n+\alpha+\beta+1)(2n+\alpha+\beta)}.\nonumber
\end{align*}
The Jacobi polynomials are orthogonal over interval $[-1, 1]$ with inner product  \cite{olver2010nist},
\begin{align*}
&\int_{-1}^1 P_{n}^{(\alpha,\beta)}(\lambda)P_{k}^{(\alpha,\beta)}(\lambda)(1-\lambda)^\alpha (1+\lambda)^\beta d\lambda \\
&=\frac{2^{\alpha+\beta+1}\Gamma(n+\alpha+1)\Gamma(n+\beta+1)}{(2n+\alpha+\beta+1)\Gamma(n+\alpha+\beta+1)n!}\delta_{nk}.
\end{align*}
Many polynomials such as Chebyshev, Legendre and Gegenbauer polynomials defined in $[-1,1]$ are the special cases of the Jacobi polynomials \cite{olver2010nist}. 

The eigenvalue $\lambda$ of the LB-operator ranges over 
$[0,\infty)$.
Expanding the exponential weight $e^{-\lambda\sigma}$ by the Jacobi polynomials may not be able to provide a good fit outside the interval $[-1,1]$.
Hence, we shift and scale 
Jacobi polynomials with parameter $b>0$ 
\begin{equation}
 \overline{P}^{(\alpha,\beta)}_n(\lambda)=P^{(\alpha,\beta)}_n\left(\frac{2\lambda}{b}-1\right), \label{eq:Pbar}
 \end{equation}
 which are  orthogonal over $[0,b]$. 
Then,  $e^{-\lambda\sigma}$ is expanded in terms of $\overline{P}^{(\alpha,\beta)}_n$.

\begin{theorem}\label{thm:jacobi}
The Jacobi polynomial expansion of the solution to heat diffusion \eqref{eq:cauchy}  is given by
\begin{equation}
K_{\sigma}* f=\sum_{n=0}^\infty c_{\sigma,n}\overline{P}^{(\alpha,\beta)}_n \left(\Delta\right) f,
\end{equation}
where the coefficients $c_{\sigma,n}$ have the closed-form solution
{\small
\begin{equation*}
c_{\sigma,n}=\frac{\Gamma(\alpha+\beta+n+1)}{\Gamma(\alpha+\beta+2n+1)}
(-b\sigma)^n {}_1 F_1\left( 		\!\!\!
\begin{array}{c}
\beta+n+1\\ 
\alpha+\beta+2n+2\end{array}		\!\!
; -b\sigma \right),
\end{equation*}}
and  $_pF_q$ is the generalized hypergeometric function \cite{olver2010nist}.
\end{theorem}

\begin{proof}
We first derive the expansion of $e^{-\lambda\sigma}$ using the Jacobi polynomials $P^{(\alpha,\beta)}_n$. The algebraic derivation will show that the expansion of $e^{-\lambda\sigma}$  is given by
{\small
\begin{equation*}
e^{-\lambda\sigma}		\!	=	\!
\sum_{n=0}^\infty\gamma_n (-2\sigma)^n e^{\sigma}  {}_1 F_1 
\left(			\!\!\! 
\begin{array}{c}
\beta+n+1\\ 
\alpha+\beta+2n+2\end{array}			\!\!
;-2\sigma 			\!
\right)P^{(\alpha,\beta)}_n(\lambda), 
\end{equation*}}
where $\gamma_n=\frac{\Gamma(\alpha+\beta+n+1)}{\Gamma(\alpha+\beta+2n+1)}$, and $_pF_q$ is the generalized hypergeometric function \cite{olver2010nist,ismail2005classical}.
The expansion is only valid in interval $[-1, 1]$. To obtain the expansion of $e^{-\lambda\sigma}$ in terms of
the shifted and scaled Jacobi polynomial \eqref{eq:Pbar},
we  replace $\lambda$ by $\frac{2\lambda}{b}-1$ and $\sigma$ by $\frac{b\sigma}{2}$ and expand $e^{-\lambda\sigma+\frac{b\sigma}{2}}$ as
\begin{equation*}	
\sum_{n=0}^\infty\gamma_n (-b\sigma)^n e^{\frac{b\sigma}{2}} {}_1 F_1
\left(			\!\!
\begin{array}{c}
\beta+n+1\\ 
\alpha+\beta+2n+2\end{array}			\!\!
;-b\sigma 	
\right) 
\overline{P}^{(\alpha,\beta)}_n(\lambda).
\end{equation*}
We divide the both sides of the equation by $e^{\frac{b\sigma}{2}}$, and the expansion of $e^{-\lambda\sigma}$ follows.
\end{proof}

\noindent{\textbf{Chebyshev polynomials.}}
The Chebyshev polynomials $T_n(\lambda)=\cos(n\cos^{-1}\lambda)$ defined  in interval $[-1,1]$ are the special cases of the Jacobi polynomials \cite{olver2010nist},
\begin{equation}\label{eq:cheby_jacobi}
{T}_n(\lambda)=\frac{4^n(n!)^2}{(2n)!} {P}_n^{(-\frac{1}{2},-\frac{1}{2})}(\lambda).
\end{equation}
The Chebyshev polynomials satisfy the recurrence relation  \eqref{eq:recurrence} with parameters $A_n=2-\delta_{n0}$, $B_n=0$ and $C_n=-1$.
Similar to using the shifted and scaled Jacobi polynomials in Theorem \ref{thm:jacobi},
we shift and scale the Chebyshev polynomials to  
\begin{equation*}
\overline{T}_n(\lambda)=T_n\left(\frac{2\lambda}{b}-1\right)
\end{equation*}
for the expansion of exponential weight over interval $[0,b]$.

\begin{theorem}\label{thm:cheby}
The Chebyshev  polynomial expansion of the solution to heat diffusion  \eqref{eq:cauchy} is given by
\begin{equation*}
K_{\sigma}*f=\sum_{n=0}^\infty c_{\sigma,n} \overline{T}_n\left(\Delta\right) f,
\end{equation*}
where the coefficients $c_{\sigma,n}$ have the closed-form solution
\begin{equation*}
c_{\sigma,n}=(2-\delta_{n0})(-1)^n e^{-	\frac{b\sigma}{2}} I_n\left(\frac{b\sigma}{2}\right),
\end{equation*}
and $I_n$ is the modified Bessel function of the first kind \cite{olver2010nist}.
\end{theorem}

\begin{proof} We provide two different proofs. The first proof is based on Theorem \ref{thm:jacobi}.
The Chebyshev polynomial is a special case of the Jacobi polynomial \eqref{eq:cheby_jacobi}, and thus their shifted and scaled versions have the relation 
$\overline{T}_n(\lambda)=\frac{4^n(n!)^2}{(2n)!}\overline{P}_n^{(-\frac{1}{2},-\frac{1}{2})}(\lambda)$.
Identifying $\alpha=\beta=-\frac{1}{2}$ in Theorem \ref{thm:jacobi}  and noting $\frac{\Gamma(\alpha+\beta+n+1)}{\Gamma(\alpha+\beta+2n+1)}=1$ when $n=0$,  we have
\begin{equation}\label{eq:proof2_1}
c_{\sigma,n} = \frac{2-\delta_{n0}}{2^{2n}n!}
(-b\sigma)^n {}_1 F_1\left( \begin{array}{c}
n+1/2\\ 
2n+1\end{array}; -b\sigma \right).
\end{equation}
The modified Bessel function is closely related to the 
generalized hypergeometric function \cite{olver2010nist},
\begin{equation}\label{eq:proof2_2}
I_n(z)=\frac{z^ne^{\pm z}}{2^n n!} {}_1 F_1 \left( \begin{array}{c}
n+1/2\\ 
2n+1\end{array}; \mp 2z \right).
\end{equation}
Substitute $I_n(z)$ in  \eqref{eq:proof2_2} with $z=\frac{b\sigma}{2}$ for the term ${}_1 F_1$ in  \eqref{eq:proof2_1},
and the result follows.

The second proof is based on the generating function of the modified Bessel functions \cite{olver2010nist}:
\begin{equation}\label{eq:proof2_3}
e^{z\cos\theta}=I_0(z)+2\sum_{n=1}^\infty I_n(z) \cos(n\theta).
\end{equation}
We use the generating function to develop the relation between exponential function and the Chebyshev polynomials.
Let $\theta=\cos^{-1}\lambda$,  and then \eqref{eq:proof2_3} can be rewritten in terms of the Chebyshev polynomials $T_n(\lambda)=\cos(n\cos^{-1}\lambda)$,
\begin{equation}\label{eq:proof2_4}
e^{z\lambda}=I_0(z)T_0(\lambda)+2\sum_{n=1}^\infty I_n(z) T_n(\lambda),
\end{equation}
where $T_0(\lambda)=1$.
Replacing $\lambda$ by $\frac{2\lambda}{b}-1$ and identifying $z=-\frac{b\sigma}{2}$ in \eqref{eq:proof2_4} give the 
expansion of $e^{-\lambda\sigma+\frac{b\sigma}{2}}$ in terms of the shifted and scaled Chebyshev polynomials $\overline{T}_n$:
\begin{equation*}
e^{-\lambda\sigma+\frac{b\sigma}{2}}=I_0\left(-\frac{b\sigma}{2}\right)\overline{T}_0(\lambda)+2\sum_{n=1}^\infty I_n\left(-\frac{b\sigma}{2}\right) \overline{T}_n(\lambda).
\end{equation*}
We divide the both sides of the equation by $e^{\frac{b\sigma}{2}}$, and the expansion of $e^{-\lambda\sigma}$ follows.
Note that  $I_n\left(-\frac{b\sigma}{2}\right)=(-1)^n I_n\left(\frac{b\sigma}{2}\right)$.
\end{proof}
In numerical implementation, given the maximum eigenvalue $\lambda_{max}$ of the discrete LB-operator, 
we set $b=\lambda_{max}$ such that the Chebyshev polynomials provide good approximation of the exponential weight over $[0,\lambda_{max}]$ \cite{hammond.2011}. 
\\

\noindent{\textbf{Hermite polynomials.}}
The Hermite polynomials 
\begin{equation*}
H_n(\lambda)=(-1)^n e^{\lambda^2}\frac{d^n}{d\lambda^n}e^{-\lambda^2}
\end{equation*}
with $H_{-1}(\lambda)=0$ and $H_0(\lambda)=1$
in  $(-\infty, \infty)$ satisfy  the recurrence relation   \eqref{eq:recurrence} with parameters   \cite{olver2010nist}
\begin{equation*}
A_n=2, B_n=0, C_n=-2n.
\end{equation*}
The orthogonal condition of the Hermite polynomials \cite{olver2010nist} is given by
\begin{equation*}
\int_{-\infty}^\infty H_n(\lambda)H_m(\lambda)e^{-\lambda^2} d\lambda=\sqrt{\pi}2^nn!\delta_{nm}.
\end{equation*}

\begin{theorem}\label{thm:hermite}
The Hermite  polynomial expansion of the solution to heat diffusion  \eqref{eq:cauchy}  is given by
\begin{equation*}
K_{\sigma}*f=\sum_{n=0}^\infty c_{\sigma,n}{H}_n\left(\Delta\right) f,
\end{equation*}
where the coefficients $c_{\sigma,n}$ have the closed-form solution
\begin{equation*}
c_{\sigma,n}=\frac{1}{n!}\left(\frac{-\sigma}{2}\right)^n e^{\frac{\sigma^2}{4}}.
\end{equation*}	
\end{theorem}

\begin{proof} It follows that the  expansion of $e^{-\lambda\sigma}$ in terms of the Hermite polynomials has coefficients 
\begin{equation*}
c_{\sigma,n}=\frac{1}{\sqrt{\pi}2^nn!}\int_{-\infty}^\infty e^{-\lambda\sigma}  H_n(\lambda)e^{-\lambda^2}  d\lambda.
\end{equation*}
The closed-form solution of the expansion coefficients can be derived through the integral property involving the Hermite polynomials
$\int_{-\infty}^\infty e^{-(\lambda-y)^2}  H_n(\lambda)d\lambda=\sqrt{\pi}2^n y^n$
\cite{gradshteyn2007table} with $y=-\frac{\sigma}{2}$.

The statement can be also proved using the exponential generating function \cite{olver2010nist},
\begin{equation*}
e^{2\lambda z-z^2}=\sum_{n=0}^\infty\frac{z^n}{n!}H_n(\lambda).
\end{equation*}
Here, it is used to derive the expansion coefficients by dividing the both sides of the equation by $e^{-z^2}$ and then 
identifying $z=-\frac{\sigma}{2}$.
\end{proof}

\noindent{\textbf{Laguerre  polynomials.}}
The Laguerre polynomials $L_n$ satisfy  the recurrence relation  \eqref{eq:recurrence} with parameters 
\begin{equation*}
A_n=-\frac{1}{n+1}, B_n=\frac{2n+1}{n+1}, C_n=-\frac{n}{n+1}
\end{equation*}
and $L_{-1}(\lambda)=0$ and $L_0(\lambda)=1$ in $[0,\infty)$ \cite{olver2010nist}.

\begin{theorem}\label{thm:laguerre}
The Laguerre  polynomial expansion of the solution to heat diffusion  \eqref{eq:cauchy}  is given by
\begin{equation*}
K_{\sigma}*f =\sum_{n=0}^\infty c_{\sigma,n} {L}_n({\Delta}) f,
\end{equation*}
where the coefficients $c_{\sigma,n}$ have the closed-form solution
\begin{equation*}
c_{\sigma,n}=\frac{\sigma^n}{(\sigma+1)^{n+1}}.
\end{equation*}
\end{theorem}

\begin{proof} From the orthogonal condition of the Laguerre polynomials \cite{olver2010nist},
\begin{equation*}
\int_0^\infty  L_n(\lambda) L_k(\lambda)  e^{-\lambda} d\lambda=\delta_{nk},
\end{equation*}
the  expansion of $e^{-\lambda\sigma}$ in terms of the Laguerre polynomials has coefficients given as the inner product of $e^{-\lambda\sigma}$ and  $ L_n$:
\begin{equation*}
c_{\sigma,n}=\int_{0}^\infty e^{-\lambda\sigma}  L_n(\lambda)e^{-\lambda}   d\lambda.
\end{equation*}
The closed-form solution of the expansion coefficients can be derived through the integral property
$\int_{0}^\infty e^{-\lambda y}  L_n(\lambda)d\lambda=(y-1)^n y^{-n-1}$
\cite{gradshteyn2007table} with $y=\sigma+1$.

Alternately, we can prove the theorem using the exponential generating function of the Laguerre polynomials \cite{olver2010nist},
\begin{equation*}
\frac{1}{1-z}e^{-\frac{\lambda z}{1-z}}=\sum_{n=0}^\infty z^n L_n(\lambda).
\end{equation*}
Multiply the both sides of the equation by $1-z$. Let $\frac{z}{1-z}=\sigma$, i.e.,  $z=\frac{\sigma}{\sigma+1}$, and then the expansion of $e^{-\lambda\sigma}$ follows.
\end{proof}

\subsection{Numerical Implementation}
The MATLAB code is available at \url{http://www.stat.wisc.edu/~mchung/chebyshev}.

{\em Expansion degree.} The expansion degree $m$  is empirically determined to the sufficiently small MSE. 
Fig. \ref{fig:hippo_allmethods} displays the heat diffusion on the left hippocampus surface mesh with 2338 vertices and 4672 triangles, with diffusion time  $\sigma=1.5$  and expansion degree $m=100$. The reconstruction error is measured  by the mean squared error (MSE) between the polynomial approximation method and  the original surface mesh. Although all the methods converged with less than degree $m=100$, the Chebyshev approximation method converges the fastest. The Chebyshev polynomials will be mainly used through the paper but other polynomials can be similarly applied.

In this study, we adopted the LB-operator discretization \cite{tan.2015} for the proposed method. This LB-operator discretization differs from our previous cotan discretization used in the FEM based diffusion solver \cite{chung.2003.cvpr,chung.2004.ISBI} and LB-basis computation \cite{qiu.2006}. To rule out the potential accuracy differences caused by different discretization methods, the LB-operator discretization in the FEM based diffusion solver \cite{chung.2003.cvpr,chung.2004.ISBI} and the LB-eigenfunction approach \cite{seo.2010.MICCAI,chung.2015.MIA} was replaced by \cite{tan.2015} for a fairer comparison.

\begin{figure}[t]
\centering
\includegraphics[width=1\linewidth]{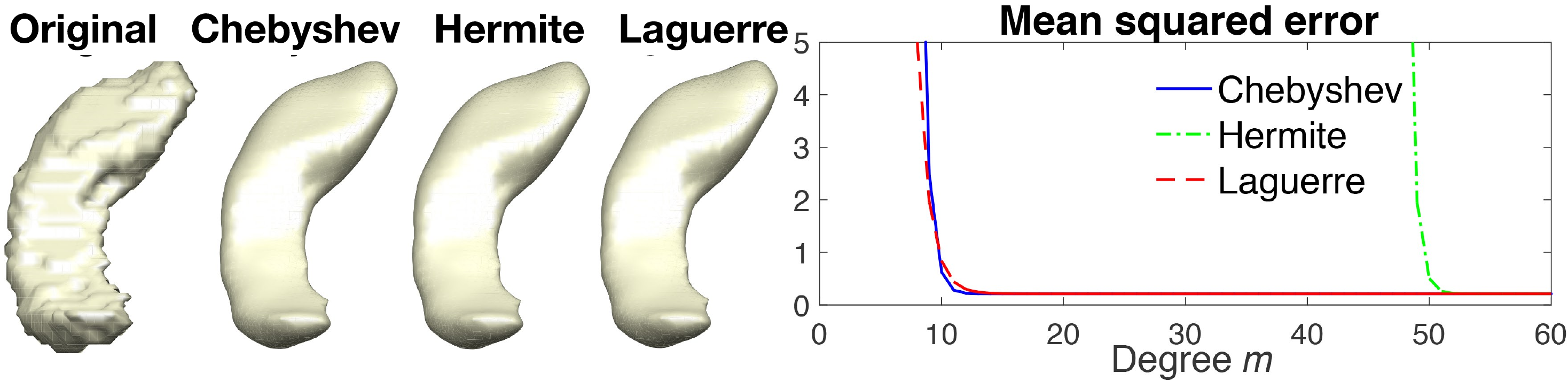}
\caption{Left: left hippocampus surface, heat diffusion with  $\sigma=1.5$ using the Chebyshev, Hermite and Laguerre approximation methods with degree $m=100$. Right: MSE between the original surface and the polynomial approximation methods for different expansion degree $m$.
The Hermite approximation method has the slowest convergence, while the Chebyshev method converges slightly faster than the Laguerre method. But the all the methods converge quickly with degree $m=100$. 
}
\label{fig:hippo_allmethods}
\end{figure}

The LB-operator  is discretized in a triangle mesh via the cotan formulation  as 
$\Delta_{ij}= C_{ij} /A_i,$
where $A_i$ is the estimated area at vertex $p_i$, and $C=(C_{ij})$ is the global coefficient matrix \cite{chung.2001.diffusion,chung.2004.ISBI,chung.2015.MIA,qiu.2006,tan.2015}. The construction of
$C_{ij}$ is as follows.
Let $T_{ij}^+$ and $T_{ij}^-$ be the two triangles sharing the same vertex $p_i$ and its neighboring vertex $p_j$.
Let the two angles opposite to the edge connecting $p_i$ and $p_j$ be $\phi_{ij}$ and $\theta_{ij}$ respectively for $T_{ij}^+$ and $T_{ij}^-$ (Fig. \ref{fig:cotan}-left).
The off-diagonal entries of the global coefficient matrix are 
$C_{ij}=-(\cot\theta_{ij}+\cot\phi_{ij})/2$
if $p_i$ and $p_j$ are adjacent and $C_{ij}=0$ otherwise.
The diagonal entries are  $C_{ii}=-\sum_{j} C_{ij}$.

\begin{figure}[t]
\centering
\includegraphics[width=0.8\linewidth,clip=true]{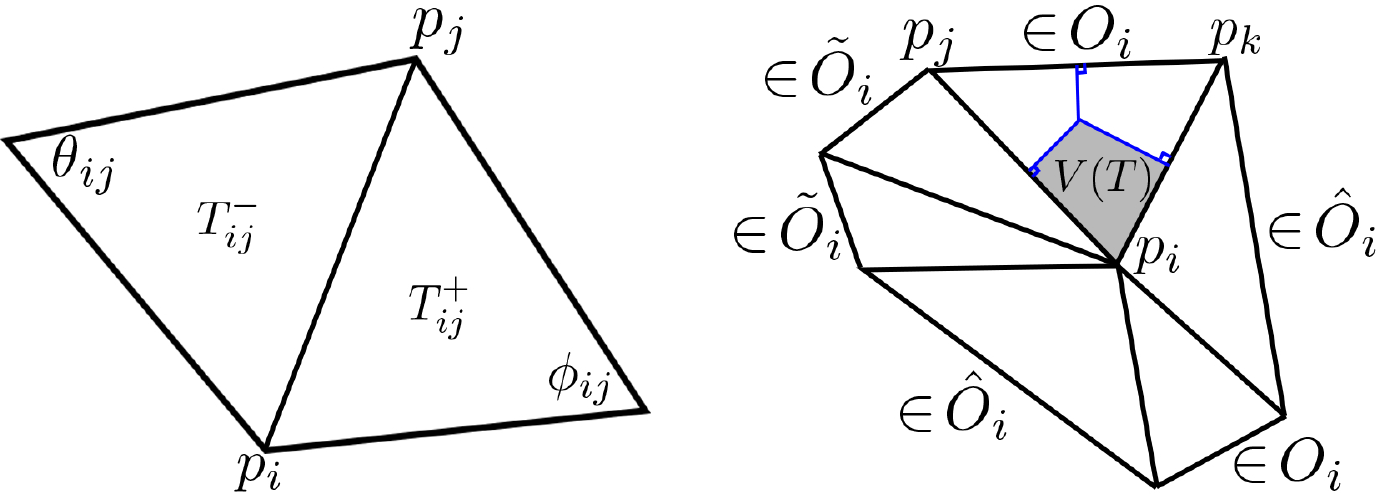}
\caption{Left: $\phi_{ij}$ and $\theta_{ij}$, angles opposite to the edge connecting $p_i$ and $p_j$  corresponding to $T_{ij}^+$ and $T_{ij}^-$, used for computing the global coefficient matrix ${\mathbf C}=(C_{ij})$.
Right: Computation  of area $A_i$ at vertex $p_i$. 
The neighboring triangles are decomposed into three sets: $O_i$ is the set of nonobtuse triangles, $\hat{O}_i$ is the set of obtuse triangles with obtuse angle at $p_i$, and $\tilde{O}_i$  is the set  of obtuse triangles with nonobtuse angle at $p_i$.}
\label{fig:cotan}
\end{figure}

For the area $A_i$, we adopt the computation in \cite{tan.2015,meyer2003discrete}.
At each vertex $p_i$, the neighboring triangles are separated into three sets: $O_i$ is the set of nonobtuse triangles, $\hat{O}_i$ is the set of obtuse triangles with obtuse angle at $p_i$, and $\tilde{O}_i$  is the set  of obtuse triangles with nonobtuse angle at $p_i$ (Fig. \ref{fig:cotan}-right). Then $A_i$ is computed as
\begin{equation*}
A_i=\sum_{T\in O_i}V(T)+ \frac{1}{2} \sum_{T\in \hat{O}_i} A(T) +\frac{1}{4} \sum_{T\in \tilde{O}_i} A(T),
\end{equation*}
where $V(T)$ is the Voronoi region (gray area) computed following  \cite{meyer2003discrete}. Let $p_j$ and $p_k$ denote the other two vertices of $T$ with angles $\angle p_j$ and $\angle p_k$ and edge lengths  $|p_ip_j|$ and $|p_ip_k|$. Then, the Voronoi region area $V(T)$ at $p_i$ is given by
$\frac{1}{8}(|p_ip_j|^2\cot\angle p_k+|p_ip_k|^2\cot\angle p_j)$ (gray area of Fig. \ref{fig:cotan}-right). The computation of $A(T)$ is done using the Heron's formula involving the three edge lengths of $T$. A simpler cotan discretization in \cite{chung.2004.ISBI,qiu.2006, chung.2015.MIA} can be also used.

{\em Iterative kernel smoothing.} We can obtain diffusion related multiscale features at different time points by iteratively performing heat kernel smoothing. Instead of  applying the  polynomial approximation separately for each  $\sigma$, the computation can also be realized in an iterative fashion. The solution to heat diffusion with larger diffusion time can be broken into iterative heat kernel convolutions with smaller diffusion time \cite{chung.2015.MIA},
\begin{equation*}
K_{\sigma_1 +\sigma_2  + \cdots + \sigma_m} \ast f = K_{\sigma_1} \ast K_{\sigma_2}\ast \cdots \ast K_{\sigma_m} \ast f.
\end{equation*}
Thus, if we compute $K_{0.25} \ast f$, then $K_{0.5} \ast f$ can be simply computed as two repeated kernel convolutions, $K_{0.25} \ast (K_{0.25} \ast f)$.
Heat diffusion with much larger diffusion time can be done similarly. Fig. \ref{fig:hippo_iterative} displays heat diffusion with $\sigma=0.25$, 0.5, 0.75 and 1 realized by iteratively applying the Chebyshev approximation method with  $\sigma=0.25$ sequentially four times. As $\sigma$ increases, we are smoothing the surface more smoothly and MSE increases.

\begin{figure}[t]
\centering
\includegraphics[width=1\linewidth]{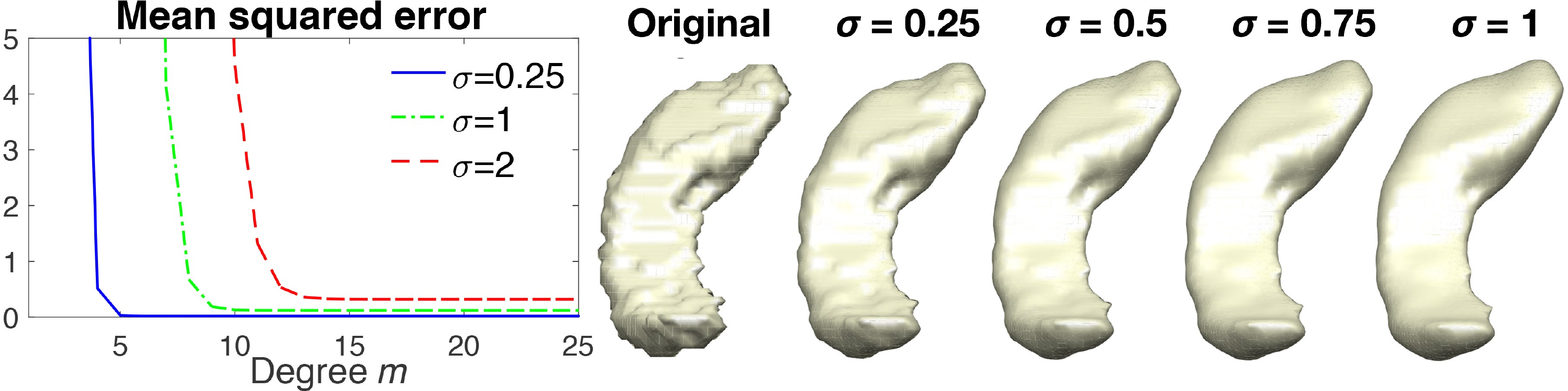}
\caption{We iteratively applied Chebyshev approximation method with $\sigma=0.25$ four times to the  left  hippocampus surface mesh coordinates to obtain heat diffusion with $\sigma=0.25$, 0.5, 0.75 and 1. As $\sigma$ increases, we are smoothing the surface more and MSE increases.}
\label{fig:hippo_iterative}
\end{figure}

\subsection{Validation}
We compared  the Chebyshev method against the FEM based diffusion solver \cite{chung.2003.cvpr,chung.2004.ISBI} and the LB-eigenfunction approach \cite{seo.2010.MICCAI,chung.2015.MIA} on the unit sphere $S^2$, where the ground truth can be analytically obtained by the spherical harmonics (SPHARM) $Y_{lm}$, which are the eigenfunctions of the LB-operator with eigenvalues
$l(l+1)$. Given surface data $f$ on the sphere
\begin{equation}\label{eq:SPHARM}
f(p)=\sum_{l=0}^{\infty} \sum_{m=-l}^l f_{lm}Y_{lm}(p),  \ \ \ p \in S^2.
\end{equation}
The heat kernel convolution at time $\sigma$ is given as \cite{chung.2007.TMI}
\begin{equation}\label{eq:SPHARM_heat}
g(p, \sigma)=\sum_{l=0}^\infty \sum_{m=-l}^l e^{-l(l+1)\sigma}f_{lm}Y_{lm}(p),
\end{equation}
where $f_{lm}=\int_{S^2} f(p)Y_{lm}(p) d\mu(p).$

\begin{figure}[t]
\centering
\includegraphics[width=1\linewidth]{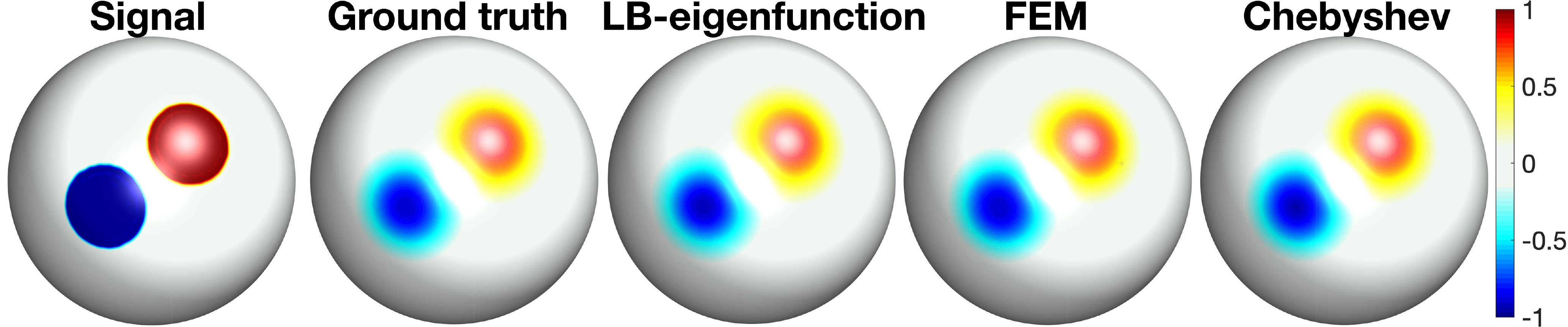}
\caption{Signal (initial condition in diffusion) and ground truth of heat diffusion with $\sigma=0.01$ and 163842 mesh vertices constructed from degree 100 SPHARM. The LB-eigenfunction approach with 210 eigenfunctions, FEM based diffusion solver
with 405 iterations, and Chebyshev approximation method with degree 45 have	 similar reconstruction error (MSE  about $10^{-5}$).} 
\label{fig:truth}
\end{figure}	

{\em Ground truth.} Assign value $1$ within one circular region, $-1$ within the other circular region, and all other regions were assigned value 0 on the spherical meshes with 2562, 10242, 40962, 163842, 655362 and  2621442  vertices (Fig. \ref{fig:truth}). We fitted the above signal using  SPHARM with degree $l= 100$, which is high enough degree to provide numerical accuracy up to 4 decimal places in terms of  MSE. The above signals were smoothed with $\sigma=0.01$ using \eqref{eq:SPHARM_heat} and taken as the ground truth.

We applied the three methods with different  $\sigma$ values (0.005, 0.01, 0.02 and 0.05). Fig. \ref{fig:truth} displays the result of the LB-eigenfunction approach with 210 eigenfunctions, FEM based diffusion solver with 405 iterations, and Chebyshev approximation method with 45 degree that achieved the  similar reconstruction error of about $10^{-5}$ MSE.

{\em Computational run time over mesh sizes.}
To achieve the similar reconstruction error, the FEM based diffusion solver and Chebyshev approximation method need more iterations and  higher degree  for larger meshes, while the LB-eigenfunction approach  is nearly unaffected by the mesh size (Fig. \ref{fig:time_vertices}-left). Fig. \ref{fig:time_vertices}-right displays the computational time of the three methods at the similar accuracy (MSE about $10^{-5}$).

\begin{figure}[t]
\centering
\includegraphics[width=1\linewidth]{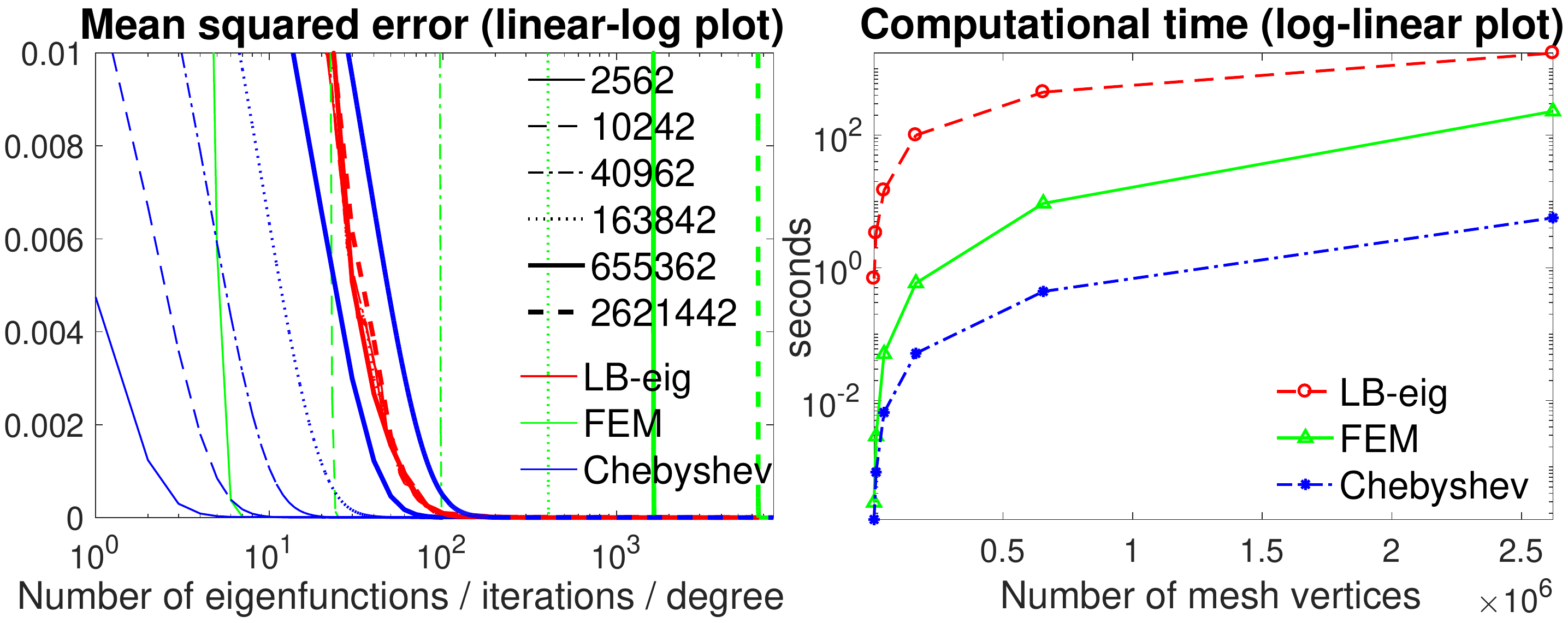}
\caption{Left: MSE of the LB-eigenfunction approach, FEM based diffusion solver and Chebyshev polynomial approximation method against the ground truth ($\sigma=0.01$) with different number of eigenfunctions, iterations and expansion degree respectively on the unit sphere with 2562, 10242, 40962, 163842, 655362 and  2621442 mesh vertices. Right: the computational time versus number of mesh vertices at the similar reconstruction error (MSE about $10^{-5}$).}
\label{fig:time_vertices}
\end{figure}

\begin{figure}[t]
\centering
\includegraphics[width=1\linewidth]{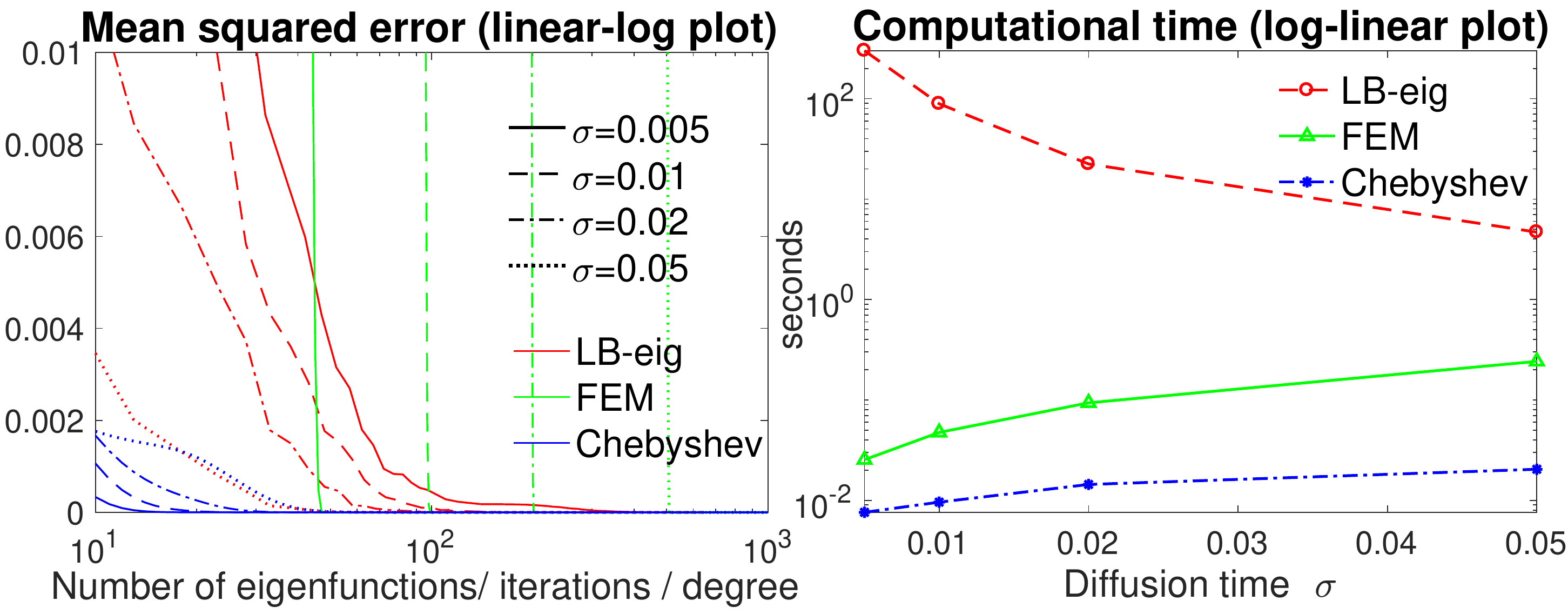}
\caption{	
Left: MSE of the LB-eigenfunction approach, FEM based diffusion solver and Chebyshev approximation method against the ground truth with different number of eigenfunctions, iterations and expansion degree respectively.
The diffusion time $\sigma=0.005$, 0.01, 0.02 and 0.05  and unit sphere with  40962 mesh vertices  were used.
Right: the computational time versus $\sigma$ at similar accuracy (MSE about $10^{-7}$).}
\label{fig:time_sigma}
\end{figure}

{\em Computational run time over diffusion times.} 
The computational run time for different $\sigma$ (0.005, 0.01, 0.02, 0.05) with fixed spherical mesh resolution (40962 vertices) was also investigated. To achieve the similar reconstruction error, the FEM based diffusion solver and Chebyshev expansion method need more iterations and higher degree for larger $\sigma$, while the LB-eigenfunction approach requires less number of eigenfunctions (Fig. \ref{fig:time_sigma}-left). Fig. \ref{fig:time_sigma}-right displays the computational run time over  $\sigma$
at the same MSE of about $10^{-7}$.

From Figs. \ref{fig:time_vertices} and  \ref{fig:time_sigma}, the LB-eigenfunction method is the slowest. The polynomial approximation method is up to 40 times faster than the FEM based diffusion solver and took 5.7 seconds for $\sigma=0.01$ on the sphere with 2621442 vertices.

\section{Application}

\subsection{HCP Dataset}
We used the T1-weighted MRI of 268 females and 176 males in the Human Connectome Project (HCP) database \cite{vanessen.2012}. MRI were obtained using a Siemens 3T Connectome Skyra  scanner  with a 32-channel head coil. The details on image acquisition parameters and image processing can be found in \cite{glasser2013minimal,smith.2013}.

A bias field correction was performed, and the T1-weighted image was registered to the MNI space with a FLIRT affine and then a FNIRT nonlinear registration \cite{jenkinson.2002}. The distortion- and bias-corrected T1-weighted image was then undergone the FreeSurfer's recon-all pipeline \cite{dale1999cortical,fischl2007cortical,segonne2005genetic}
that includes the segmentation of  volume into predefined structures,
reconstruction of white and pial cortical surfaces, and FreeSurfer's standard folding-based surface registration to their surface atlas. Then, the white, pial and spherical surfaces of the left and right hemispheres were produced.

\subsection{Sulcal and Gyral Curve Extraction}

The automatic sulcal curve extraction method (TRACE) \cite{lyu2010spectral, lyu.2018} was used to detect concave regions (sulcal fundi) along which sulcal curves are traced. The method consists of two main steps: (1) sulcal point detection and (2) curve delineation by tracing the detected sulcal points. For sulcal point detection, concave points are initially obtained from the vertices of the input surface mesh by thresholding mean curvatures.  The concave points are further filtered by employing the line simplification method \cite{ramer1972iterative}
that simplifies the sulcal regions without significant loss of their morphological details. For curve delineation, the selected sulcal points are connected to form a graph, and the curves are delineated by tracing shortest paths on the graph. Finally, the sulcal curves are traced over the graph by the Dijkstra's algorithm \cite{dijkstra.1959}. We use similar idea to gyral curve extraction by finding convex regions.

\begin{figure}[t]
\centering
\includegraphics[width=1\linewidth,clip=true]{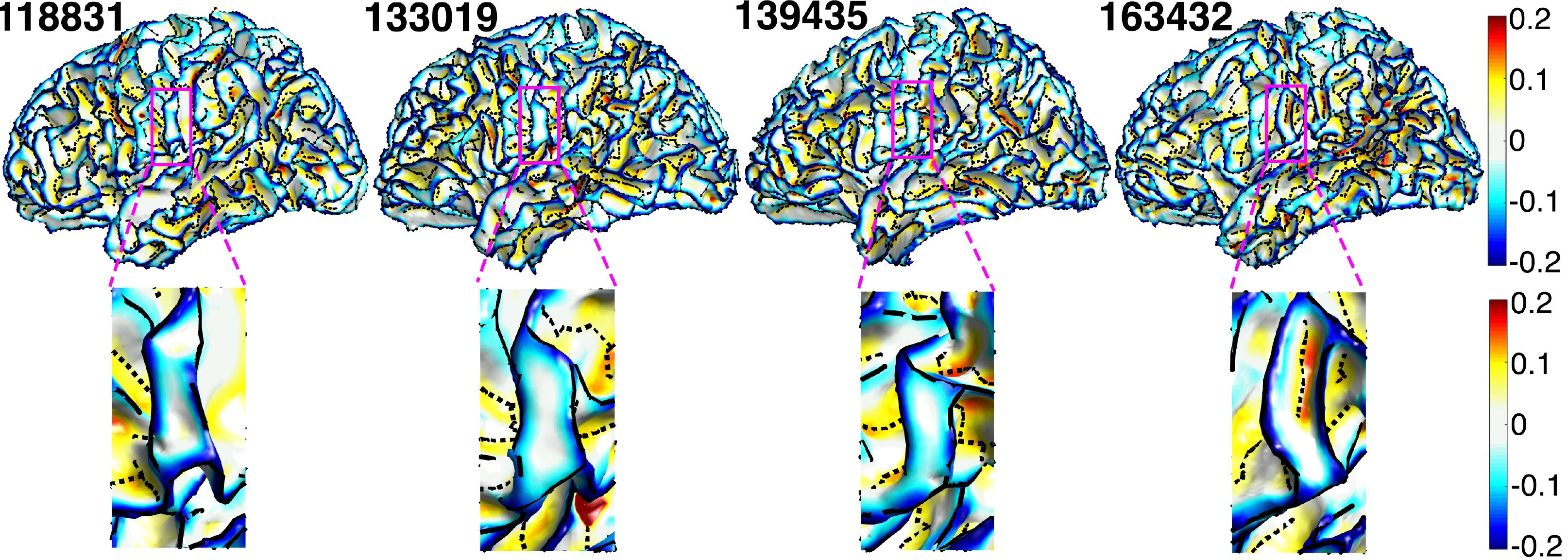}
\caption{Top: gyral curves (black solid line), sulcal curves (black dashed line), and the smoothed mean curvature  of four subjects.
Bottom: the enlarged magenta regions of the top figures showing that there is no sulcal curve between gyral curves in the left three subjects due to shallow depth or low mean curvature).}
\label{fig:nosulcal}
\end{figure}

The TRACE method only identified the major gyral and sulcal curves. Minor curves in almost flat regions like plateau or with very low curvature, shallow depth or short length  were not extracted. 
Fig. \ref{fig:nosulcal} displays the sulcal and gyral curves and the smoothed mean curvature of four subjects.
In the enlarged regions, the first three subjects have no sulcal curves between the two gyral curves due to very low mean curvature, while the fourth subject has sulcal curve in the same region because of  higher mean curvature.

\begin{table}[t]
\caption{Reproducibility and robustness to noise measured by average and Hausdorff distances (mm) (results from \cite{lyu.2018}).} 
\label{table:TRACE}
\begin{adjustbox}{width=\columnwidth,center}
\begin{tabular}{ | l | l |  l   l  l | l  l   l |}
\hline
\multirow{2}{*}{} & \multirow{2}{*}{Method} & \multicolumn{3}{c|}{Average distance}  &  \multicolumn{3}{c|}{Hausdorff distance}  \\ 
& & mean  & min & max  & mean  & min & max \\ 
\hline
\multirow{2}{*}{Reproducibility} & TRACE & 1.00 & 0.87 & 1.17 & 1.71 & 1.54 & 1.99\\ 
   						&Li {\em et al.} \cite{li2010automated} & 1.23  & 0.87 & 3.28 & 1.94 & 1.54 & 4.23\\ 
\hline
\multirow{2}{*}{Robustness} & TRACE & 1.06   & 0.99 &  1.19 & 1.82 & 1.67 & 2.06	\\ 
   						&Li {\em et al.} \cite{li2010automated} & 1.42 & 1.21 & 1.62 & 2.73 & 2.18 & 3.29\\ 
\hline
\end{tabular}
\end{adjustbox}
\end{table}

The TRACE method was validated using the Kirby reproducibility dataset with 21 T1-weighted scans \cite{landman2011multi}. The reproducibility was measured by the distance between two corresponding surfaces (scan and re-scan sessions). The robustness to noise compared to  \cite{li2010automated} was done using synthetic noisy surfaces, which were generated by adding vertex-wise random displacements to the original surfaces. The displacement at each vertex follows an independent and identically distributed uniform distribution between 0 and 1.0 mm. We used the average and Hausdorff distances \cite{huttenlocher1993comparing}. The experimental results from \cite{lyu.2018} (Table \ref{table:TRACE}) show
higher reproducibility and robustness to noise in TRACE than the existing method \cite{li2010automated}. The paired $T$-test showed significant differences between these two methods in both the average and  Hausdorff distances, with $p$-values 0.0045 and 0.003 respectively in reproducibility  and $p$-values$<10^{-16}$ in robustness. For the comparison with manually labeled primary curves \cite{joshi2010sulcal,pantazis2010comparison}, the MRIs Surfaces Curves dataset (\url{http://sipi.usc.edu/~ajoshi/MSC}) consisting of 12 subjects was used. The mean values of the average and Hausdorff distances of the 26 primary curves are 1.32 and 3.77 mm in the TRACE method, which are smaller than 1.38 and 4.20 mm in \cite{li2010automated}. Even though the paired $T$-test found no significant difference between the two methods in the average distance ($p$-value=0.0713), we found significant difference in the Hausdorff distance ($p$-value=$7.3\times10^{-6}$).

\begin{figure}[t]
\centering
\includegraphics[width=1\linewidth,clip=true]{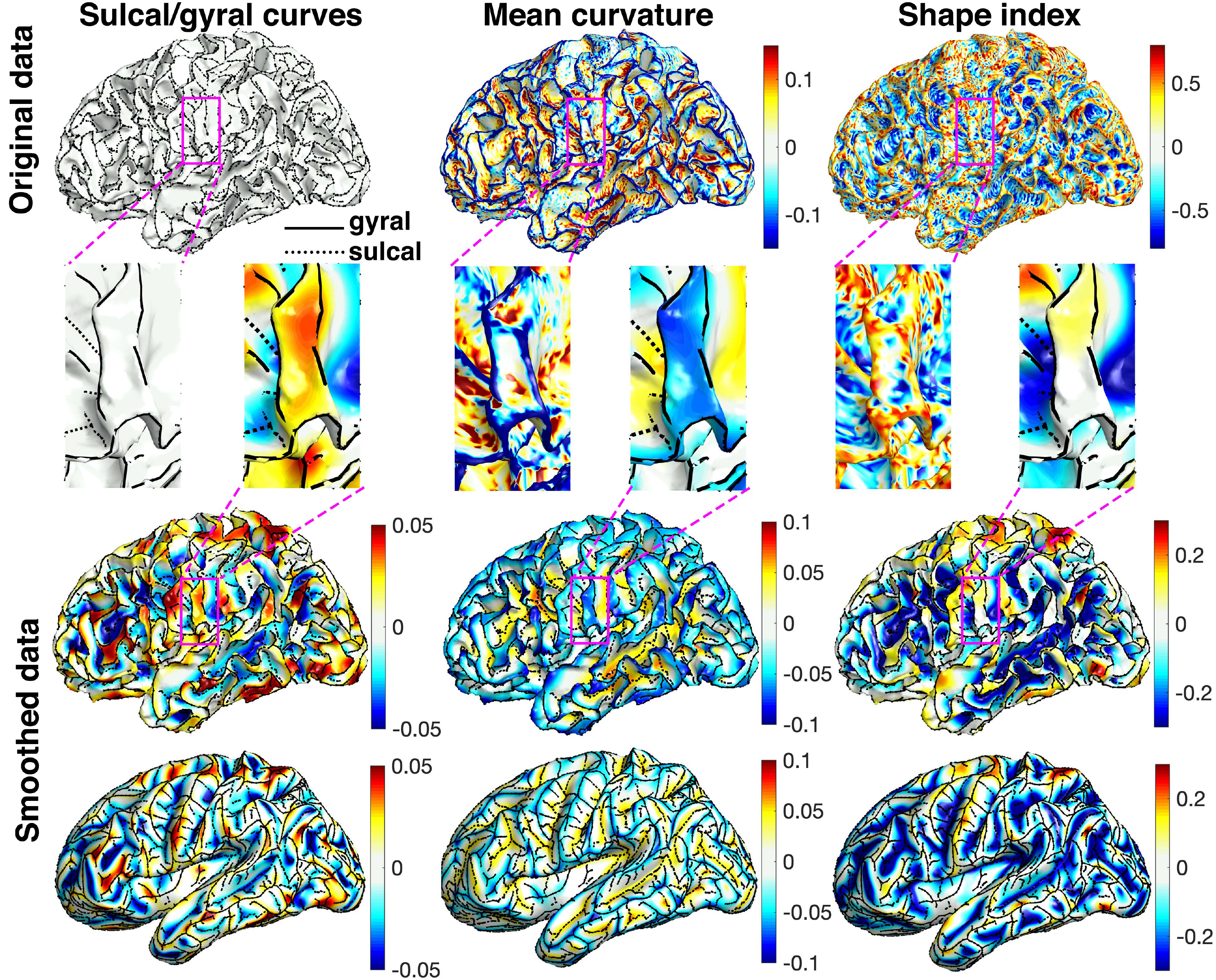}
\caption{The sulcal/gyral curves  (left),  mean curvature (middle) and  SI (right).
1st and 2nd rows:  original data displayed on the white matter surfaces and the enlarged magenta regions.
The gyral and sulcal curves are marked by solid and dashed black lines respectively and are assigned heat values 1 and -1 when smoothing. The mean curvature is positive for sulci and negative for gyri. The SI is positive for gyri and negative for sulci.
In the enlarged magenta regions, the noisy mean curvature and SI show sulcal patterns in the middle of the gyral region, which is not shown in the sulcal/gyral curve extraction method. Smoothing is done with diffusion time $\sigma=0.001$.}
\label{fig:curvature_all}
\end{figure}

\begin{figure*}[t]
\centering
\includegraphics[width=0.93\linewidth,clip=true]{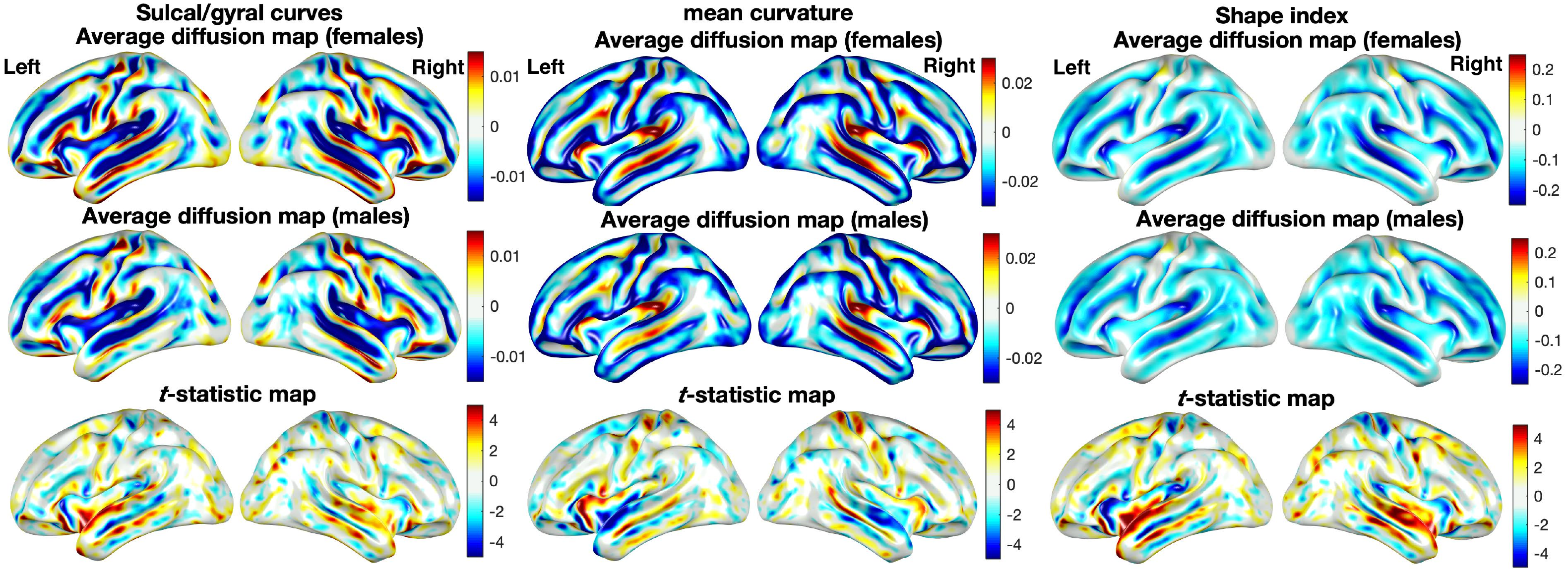}
\caption{The average diffusion maps of the  sulcal/gyral curves (left), mean curvature (middle) and SI (right) of 268 females and 176 males displayed on the average surface template. 
The $T$-statistic  maps (bottom) show the
localized sulcal and gyral pattern, mean curvature and SI differences (female - male) thresholded at $\pm 4.96$ (uncorrected $p$-value of $10^{-6}$),
having significant difference mainly in the temporal lobe.}
\label{fig:curv_ttest}
\end{figure*}

\subsection{Diffusion Maps on Sulcal and Gyral Curves} 
The junctions between sulci are highly variable \cite{mangin2004object}. A sulcus corresponding to a long elementary fold in one subject may be made up of several small elementary folds in  another subject \cite{cachia.TMI.2003}. Each subject has different number of vertices and edges in sulcal and gyral graphs, and they don't exactly match across subjects even after registration. It is difficult to directly compare such graphs at the vertex level across subjects. Thus, the proposed polynomial approximation was used to smooth out  the sulcal  and gyral curves and obtain the smooth representation of curves that enables the vertex-level comparisons.

The extracted gyral curves were assigned heat value 1, and sulcal curves were assigned heat value -1. All other parts of surface mesh vertices were assigned value 0. Then heat diffusion was performed on these values. The diffusion map values range from -1 to 1. The close to the value of 1 indicates the likelihood of the gyral curves while the close to the value of -1  indicates the likelihood of the sulcal curves. The proposed method is motivated by the voxel-based morphometry (VBM) \cite{chung.2003.NI,ashburner.2000}, where the segmented white or gray matter regions are compared in 3D volume. Due to the difficulty of exactly aligning the white or gray matter regions separately, Gaussian kernel smoothing with large bandwidth was used to mask the shape variations across subjects and approximately align the segmented regions. Also a similar approach was used in the tract-based spatial statistics (TBSS) \cite{bodini2009exploring,bach2014methodological} in analyzing white matter regions in diffusion tensor imaging  that does not exactly align across subjects.

In the numerical implementation, a sufficient large expansion degree  $m=1000$ were used. In a desktop with 4.2 GHz Intel Core i7 processor, the construction of the discrete LB-operator took 5.76 seconds, the computation of Chebyshev coefficients took $2.6\times10^{-4}$ seconds, and heat diffusion by the Chebyshev polynomials took 3.19 seconds for the both hemispheres in average. The total computation took  8.95 seconds per subject in average. The diffusion maps were then subsequently used in localizing the male and female differences. One example of  diffusion map  is displayed in Fig. \ref{fig:curvature_all}-left.

\subsection{Univariate Two-Sample $T$-Test}  
The diffusion maps with $\sigma=0.001$ were constructed for  268 females 176 males. The average diffusion maps in Fig. \ref{fig:curv_ttest}-left displays the major differences in the temporal lobe, which is responsible for processing sensory input into derived meanings for the appropriate retention of visual memory, language comprehension, and emotion association \cite{smith.2007}. 

The two-sample $T$-statistics maps are in the range of $[-6.5, 7.02]$. Any $T$-statistic with absolute value above 2.75
(red and blue regions) is considered as statistically significant  using the false discovery rate (FDR) at 0.05. If the $T$-statistic map shows high $T$-statistic value at a particular vertex, it indicates that one group has consistently more gyral curves than sulcal curves at the vertex. If we use slightly different diffusion time $\sigma$, we still obtain similar results.

We did an additional analysis using the mean curvature and shape index (SI). We estimated the curvatures and SI, which are the functions of curvature, by fitting the local quadratic  surface in the first neighboring vertices \cite{joshi1995differential,chung.2003.NI} 
\begin{equation*}
f(x_1,x_2) = \beta_0+\beta_1 x_1 + \beta_2 y_2 + \frac{1}{2}\beta_3 x_1^2 + \beta_4 x_1x_2 + \frac{1}{2}\beta_5 x_2^2.
\end{equation*}
The curvature and SI are expected to be noisy and require smoothing to increase statistical sensitivity and the signal-to-noise ratio \cite{chung.2001.diffusion, tosun2008geometry,joshi2009parameterization} (Fig. \ref{fig:curvature_all}). Smoothing surface data before statistical analysis is often done in various cortical surface features. Even the FreeSurfer package output the smoothed mean and Gaussian curvatures \cite{pienaar2008methodology,fischl2012freesurfer}.
Fig. \ref{fig:curvature_all} displays the results of smoothed mean curvature and SI maps.

We performed  the two-sample $T$-test on the smoothed mean curvature and SI maps (Fig. \ref{fig:curv_ttest}-middle and right). The results show significant gender difference mainly in the temporal lobe,  consistent to the findings in the proposed sulcal/gyral curve analysis (Fig. \ref{fig:curv_ttest}-left). \\

\subsection{Multivariate Two-Sample $T$-Test}
The iterative kernel convolution was used to compute the diffusion at different time points quickly.  The values of diffusion at different time points were then used in constructing the multiscale  features. In this study, we adopted 10 time points $\sigma=0.0005, 0.001, \cdots, 0.0045, 0.005$. Fig. \ref{fig:subj67_Heat} shows the diffusion maps of one representative subject. At each vertex, the multiscale diffusion features are used 
to determine the significant difference between the  females and males. We used the two-sample Hotelling's $T^2$-statistic, which is the multivariate generalization of the two-sample $T$-statistic \cite{kim2012wavelet,chung.2008.MIAR}. 
Fig. \ref{fig:HotelT2} shows the Hotelling's $T^2$-statistics and the corresponding $p$-values in the log-scale.
The heat diffusion has $T^2$-statistics in the range of  $[0.13,\ 8.2]$ with minimum $p$-value $3.4\times10^{-12}$. 
Any $T^2$-statistic above 2.28 (yellow and red regions) is considered as significant at FDR 0.05.

In comparison, we used the {\em diffusion wavelet features} \cite{coifman2006diffusion,tan.2015,hammond.2011,kim2012wavelet}
at ten different scales and showed that the proposed method can achieve similar performance in localizing signal regions as the wavelet features.  
The diffusion wavelet \cite{coifman2006diffusion,hammond.2011,tan.2015} has the similar algebraic form as the heat kernel:
\begin{equation*}
W_t(p,q)=\sum_{j=0}^{\infty} g(\lambda_j t)\psi_j(p)\psi_j(q).
\end{equation*}
The difference between the heat kernel and diffusion wavelet transform is the weight function $g$,
which determines the spectral distribution.

Compared to the heat kernel, the weight function $g$ attenuating all low and high frequencies outside the passband makes the diffusion wavelet work as a band-pass filter.
The wavelet transform transform of $f$ is then given by
\begin{equation*}
W_t \ast f (p) 
=\sum_{j=0}^{\infty}g(\lambda_j t)f_j\psi_j(p), \ \ \ \ f_j=\int_{\mathcal{M}} f(p) \psi_j (p) d\mu(p).
\end{equation*}
The proposed polynomial approximation scheme can be applied to the diffusion wavelet transform through expanding $g(\lambda t)$ by orthogonal polynomials. 
In this paper, we used the following cubic spine as $g(\lambda t)$ \cite{hammond.2011}
\begin{equation}\label{eq:WT_weight}
g(x)=\left\{\begin{array}{ll}
x_1^{-\alpha} x^\alpha, & x<x_1\\ 
-5+11x-6x^2+x^3, & x_1\leq x\leq x_2\\
x_2^{-\beta} x^\beta, &x>x_2
\end{array}\right.,
\end{equation}
where $\alpha=\beta=2$, $x_1=1$ and $x_2=2$.
The scaling parameter  $t$  controls the passband of the diffusion wavelet (Fig. \ref{fig:weight}).

Diffusion at  different diffusion time $\sigma$ and diffusion wavelets at different scaling parameter  $t$ contain different spectral information  of  input data $f$ (Fig. \ref{fig:weight}). Thus, the heat diffusion with a varying $\sigma$ and diffusion wavelet with varying $t$ provide multiscale  features of  $f$. All the heat diffusion features contain low-frequency components. If the initial  surface data  suffer from significant low-frequency noise, the diffusion wavelet transform would be more suitable. On the other hand, if most noises are in high frequencies, performance of the both methods would be similar and we do not really needs diffusion wavelet features \cite{bernal2008detection}.

\begin{figure}[t]
\centering
\includegraphics[width=1\linewidth,clip=true]{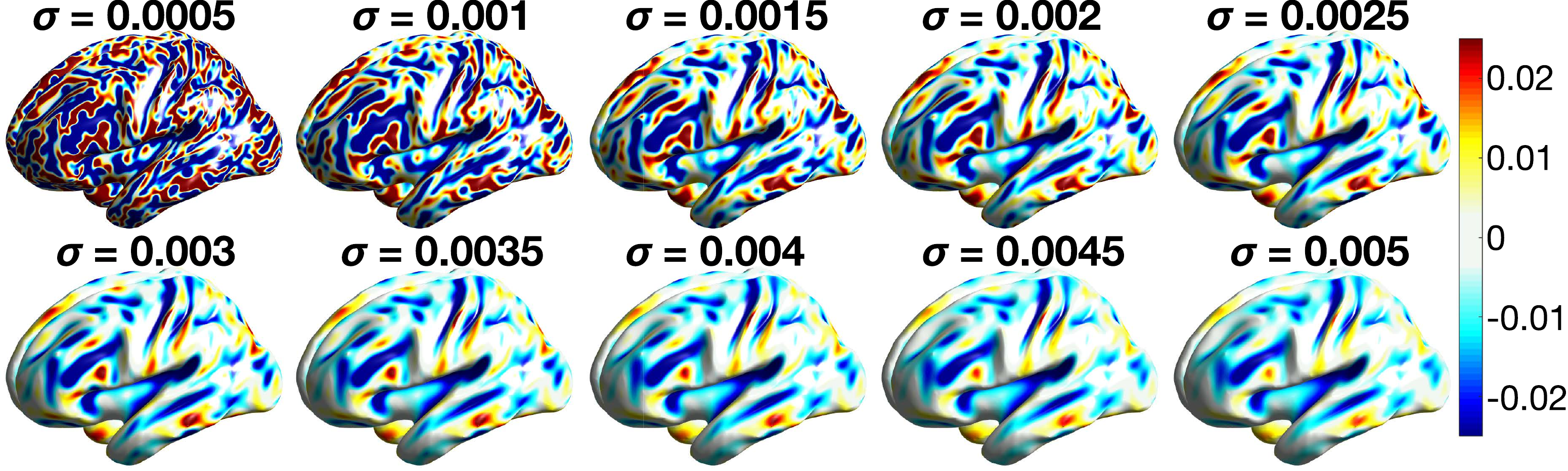}
\caption{Heat diffusion with $\sigma=0.0005, 0.001, \cdots, 0.005$ computed by the iterative convolution of subject 130114. 
At each vertex, 10 diffusion values at different time points are used in constructing the Hotelling's $T^2$-statistic to contrast males and females.}
\label{fig:subj67_Heat}
\end{figure}

\begin{figure}[t]
\centering
\includegraphics[width=1\linewidth,clip=true]{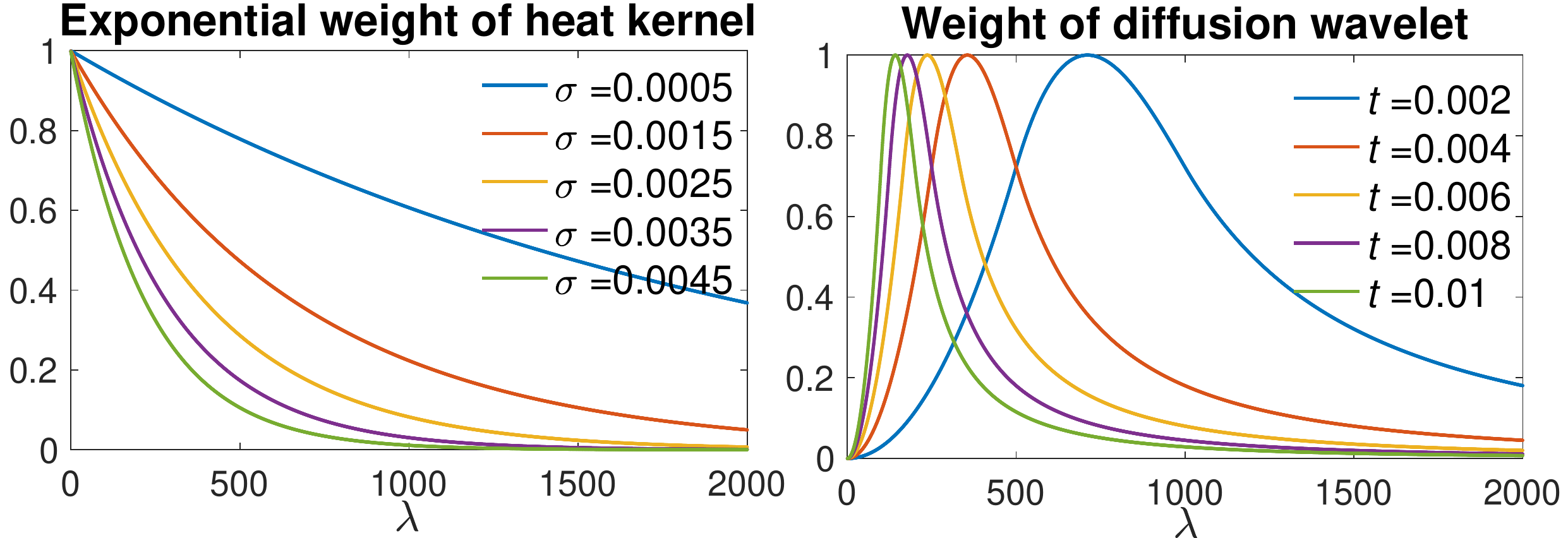}
\caption{Left: the exponential weight $e^{-\lambda\sigma}$ in heat diffusion for diffusion time $\sigma=0.0005$, $0.0015$,...,$0.0045$.
Right: the weight function $g(\lambda t)$ in diffusion wavelet transform for scaling parameter $t=0.002$, $0.004$,..., $0.01$.}
\label{fig:weight}
\end{figure}

\begin{figure}[t]
\centering
\includegraphics[width=1\linewidth,clip=true]{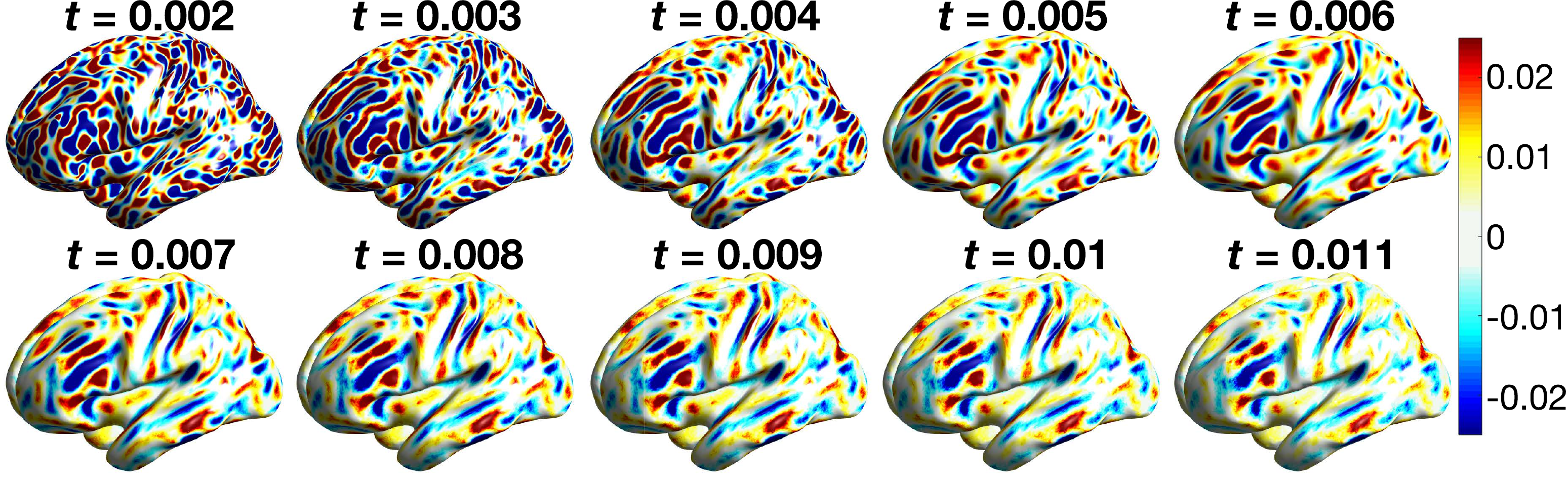}
\caption{Diffusion wavelet transform with scaling parameters $t=0.002,0.003,\cdots, 0.011$ of subject 130114. 
At each vertex, 10 diffusion wavelet transform values at different scales are used in constructing the Hotelling's $T^2$-statistic to contrast males and females.}
\label{fig:subj67_WT}
\end{figure}

\begin{figure}[t]
\centering
\includegraphics[width=1\linewidth,clip=true]{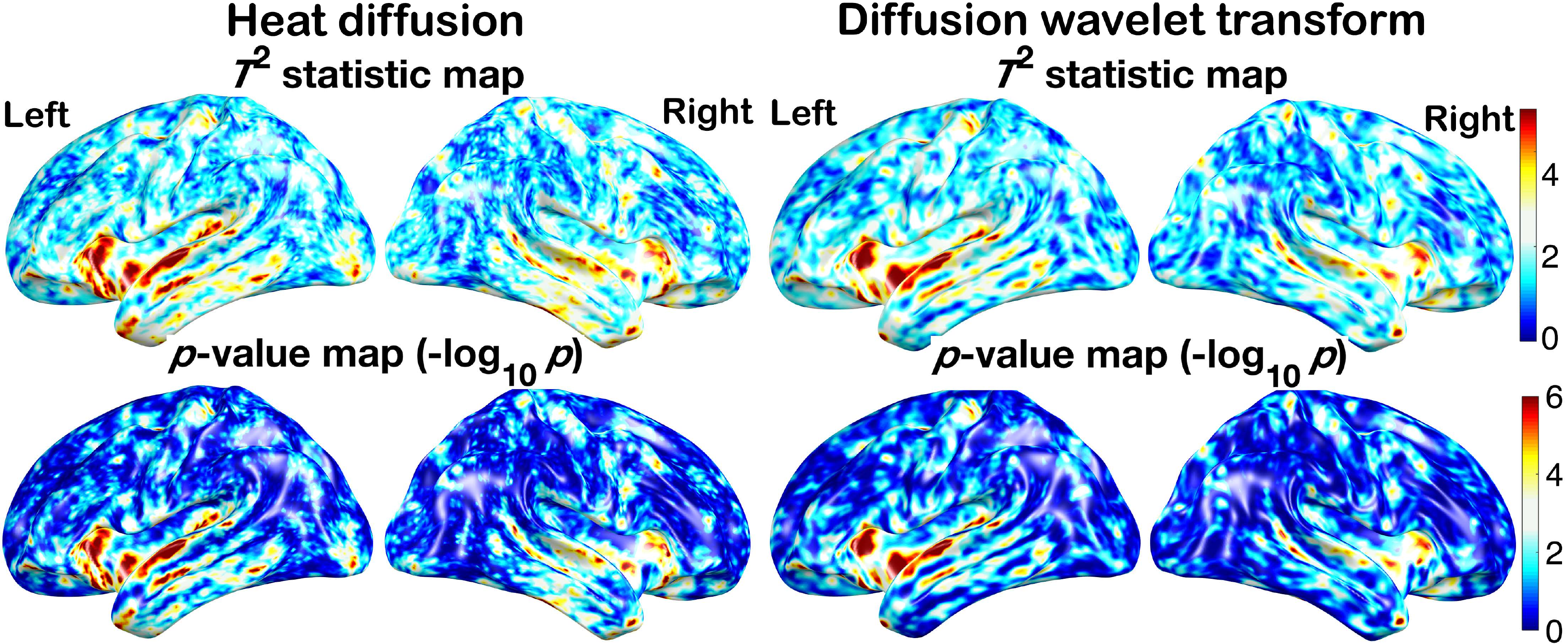}
\caption{Top: Hotelling's $T^2$-statistics of left and right hemispheres of heat diffusion maps with 10 different time points (left) and diffusion wavelet transform with 10 different scales (right). 
The $T^2$-statistic maps of left and right hemispheres showing the
localized sulcal and gyral pattern differences (female - male) thresholded at 4.9 (uncorrected $p$-value of $10^{-6}$).
Bottom: $p$-value maps of left and right hemispheres displayed in logarithmic scale show the significance of the difference at the uncorrected $p$-value of $10^{-6}$.}
\label{fig:HotelT2}
\end{figure}

In this study, we adopted 10 different values of $t=0.002,0.003,\cdots, 0.011$.
Fig. \ref{fig:subj67_WT} shows the flattened diffusion wavelet maps of one representative subject. The values of $t$ were chosen empirically to match the amount of smoothing (FWHM) in the wavelet to the amount of smoothing in heat diffusion.  
Using the two-sample Hotelling's $T^2$-statistic on the multiscale diffusion wavelet features, we also contrasted 268 females and 176 males. Fig. \ref{fig:HotelT2} shows the Hotelling's $T^2$-statistics and the corresponding $p$-values in the log-scale.
The diffusion wavelet transform has $T^2$-statistics in the range of $[0.09,\ 7.6]$ with minimum $p$-value $3.4\times10^{-11}$. For multiple comparisons, any $T^2$-statistic above 2.37  (yellow and red regions)
is considered as significant using the FDR 0.05. Although there are slight differences, the both methods show the similar localization of  sulcal and gyral graph patterns, mainly in the temporal lobe.

The exponential weight in the heat diffusion has only one parameter, i.e., the diffusion time $\sigma$, and leads to the analytic closed-form solutions to the expansion coefficients. The weight function in the diffusion wavelet transform is more complicated, and it may not possible to derive the closed-form expression for the expansion coefficients. 
The simpler weight function in heat kernel and the iterative convolution scheme lead to
faster  computational run time compared to the diffusion wavelets. 
In heat diffusion, we only needed to compute the expansion coefficients for $\sigma=0.0005$ and reused these coefficients in the  iterative convolution to obtain the other nine features. The computation of the 1000 degree expansion coefficients by the proposed closed-form solution costed only $2.6\times10^{-4}$ seconds. In the diffusion wavelet transform, due to the more complicated weight function, the 1000 degree expansion coefficients were computed numerically, which took 1.26 seconds \cite{hammond.2011,kim2012wavelet}.

\begin{figure*}[t]
\centering
\includegraphics[width=1\linewidth,clip=true]{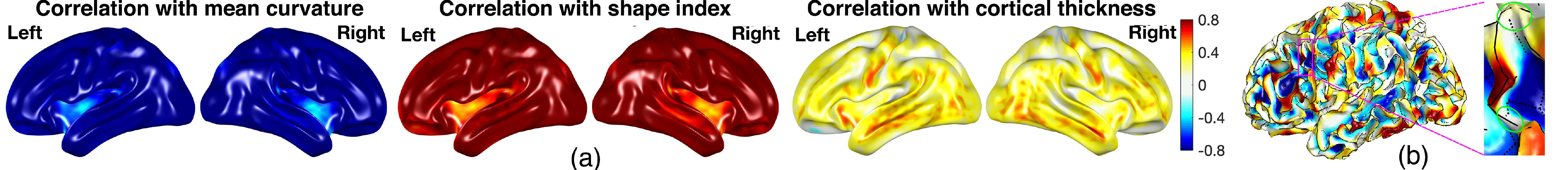}
\caption{(a) Correlations between the diffusion maps ($\sigma=0.001$) of the sulcal/gyral curves and the mean curvature, SI and cortical thickness across all subjects. We observe  strongly negative correlation to the mean curvature  ($-0.77\pm0.12$) and strongly positive correlation to the SI ($0.76\pm0.13$) and high positive correlation to cortical thickness in many regions including the temporal lobe. (b) The green-colored circles in the enlarged region show diffusion map close to 0 due to interwinding complex sulcal (dashed lines) and gyral curves (solid lines).}
\label{fig:corr_curvature}
\end{figure*}

\subsection{Comparing to Other Cortical Folding Features}
We computed the correlations between the diffusion maps  of the sulcal/gyral curves and the mean curvature, SI and cortical thickness across all subjects 
with diffusion time $\sigma=0.001$ \cite{fisher.1915,chung2011mapping} (Fig. \ref{fig:corr_curvature}).
We can observe strongly negative correlation to mean curvature  ($-0.77\pm0.12$) and strongly positive correlation to SI ($0.76\pm0.13$). Although the correlation to the cortical thickness ($0.23\pm0.20$) is not as high  as the correlation to the mean curvature and SI, many regions have correlation value larger than 0.4, especially in the temporal lobe.
Most of regions are all statistically significant after FDR correction at 0.05.

\section{Discussion and Conclusion}
The numerical computation of solving diffusion equations has been thoroughly investigated and mostly solved using the finite element method (FEM) and finite difference method (FDM) on triangulated surface meshes for many decades \cite{andrade.2001, chung.2001.diffusion,chung.2003.cvpr,chung.2004.ISBI,chung.2015.MIA, cachia.TMI.2003,joshi2009parameterization}. It is extremely difficult if not impossible to speed up the computation over existing methods any further. Utilizing the proposed spectral decomposition of heat kernel, we were able to come up with a new numerical scheme that speeds up the computation 8-40 times (depending on the mesh size)  over these existing  methods, which is a significant contribution itself.

Due to the advancement of imaging techniques, we are  beginning to see ever larger meshes. For instance, 
\cite{joshi2009parameterization} smoothed the mean curvature of triangulated cortical surfaces with 1.4 million vertices.
\cite{fan2018tetrahedron} computed the heat flux signature in  a cortical tetrahedral mesh with more than 1.5 million vertices.
\cite{kay2019critical}  used six parallel surfaces between the pial and white surfaces with 5 million vertices in modeling fMRI BOLD activity patterns at  sub-millimeter resolution.
\cite{warner2019high} generated 3D tetrahedral head and cortical surface meshes with 2.7 million vertices to build high-resolution head and brain computer models for  fMRI and EEG. Thus, the increase of computational run time would be of great interest.

In the sulcal and gyral graph pattern analysis, the sulcal and gyral curves were assigned value -1 and 1 respectfully. In the regions of higher diffusion value close to 1, there are more gyral curves than sulcal curves. In the regions of lower diffusion value close to -1, there are likely more sulcal curves than gyral curves. If a group consistently higher diffusion value in a particular region, it indicates there are likely to be more gyral curves in that region. The regions of interwinding complex sulcal and gyral curves will result in diffusion maps close to 0 (Fig. \ref{fig:corr_curvature}). Thus, the statistically significant group differences are not due to  the complexity of interwinding sulcal/gyral patterns but 
the consistent concentration of more gyral or sulcal curves.

We found that the differences are mainly in the temporal lobe, especially in the superior temporal gyrus and sulcus, which is consistent with the literature. \cite{harasty1997language} reported that females  have proportionally larger language areas compared to males, such as the superior temporal cortex and Broca's area. \cite{ochiai2004sulcal} 
reported statistically differences between males and females in  the right superior temporal sulcus and the  most posterior point and center of the left superior temporal sulcus. The significant gender differences in sulcal width and depth were reported in the superior temporal, collateral, and cingulate sulci in \cite{kochunov2005age}. Also, there were significant gender differences in the  cortical area of the left frontal lobe and in the gyrification index of the right temporal lobe \cite{yoon2009effect}. \cite{luders2005mapping} detected higher gray matter concentrations in females in the left posterior superior temporal gyrus and left inferior frontal gyrus.
\cite{crespo2011sex} found significant differences between males and females in the sulcal curvature index of the temporal and  occipital lobes. \cite{lyu2018cortical} reported that females showed higher gyrification in the superior temporal, right inferior frontal, and parieto-occipital sulcal regions. In \cite{sepehrband2018neuroanatomical}, the mean curvature of the left superior temporal sulcus was identified as a highly discriminative feature of sex classification.
The consistent result with previous studies shows that the sulcal/gyral curves are reliable cortical surface features.

Sulcal and gyral curves can be 
interpreted with respect to cortical folding.  
The cortical folding is usually measured by  the mean curvature and SI \cite{pienaar2008methodology,shimony2016comparison}. 
Our result show that the sulcal and gyral curves are almost linearly related to the existing mean curvature and SI.  This is the reason that we got the similar statistical results in all three methods. 
The cortical folding is known to correlate to cortical thickness. 
The gyri are thicker than the sulci \cite{fischl2000measuring,toro2005morphogenetic,wagstyl2015cortical}.
Observing higher diffusion value at a vertex implies that the vertex is closer to gyri than sulci, and thus larger thickness is expected at the vertex.

In the cortical growth and folding development in human fetal brains, many studies have reported changes in surface and shape features such as the curvatures, sulcal depth, gyrification index, sulcal pit based graphs and sulcal skeletons
\cite{dubois2007mapping,habas2011early,clouchoux2012quantitative,wright2014automatic,lefevre2015developmental,im2019sulcal}.
At 25 weeks, the cortical surface is still very smooth \cite{clouchoux2012quantitative}. There are few major sulcal and gyral curves, and most surface vertices will have heat diffusion values close to zero. With increasing gestational age, the cortical folds become more complex with more sulcal and gyral curves and branches, which will likely  result in higher variability in diffusion values across vertices.

The proposed general polynomial approximation of the Laplace-Beltrami (LB) operator works for an arbitrary orthogonal polynomial. The proposed polynomial expansion method speeds up the computation
compared to existing numerical schemes for diffusion equations. Our method avoids various numerical issues associated with the LB-eigenfunction method and FEM based diffusion solvers. The proposed fast and accurate scheme can be further extended to any arbitrary domain without much computational bottlenecks. Thus, the method can be easily applicable to large-scale images where the existing methods may not be applicable without additional computational resources. Beyond the sulcal and gyral graph analysis on 2D surface meshes, the proposed method can be applied to 3D volumetric meshes  \cite{wang2015novel}. 
This is left as a future study.

\bibliographystyle{IEEEtran}
\bibliography{reference.TMI.20200112.bib}

\end{document}